\documentclass[submission,copyright,creativecommons]{eptcs}

\providecommand{\event}{GandALF 2022} 
\usepackage{breakurl}             
\usepackage{underscore}           

\usepackage{amsmath,amssymb,amsthm,thmtools}
\usepackage[all]{foreign}
\usepackage{multirow}         
\usepackage{bigdelim}         
\usepackage{extarrows}        
\usepackage{braket}           
\usepackage{lipsum}
\usepackage[capitalise]{cleveref}
\usepackage{etoolbox}
\usepackage{booktabs}

\usepackage{tikz}
\usepackage{packages/graphics}
\usepackage{packages/commands}
\usepackage{packages/symbols}

\newtheorem{definition}{Definition}
\newtheorem{theorem}{Theorem}

\title{Controller Synthesis for Timeline-based Games}

\author{
  Renato Acampora
  \institute{University of Udine, Italy}
  \email{acampora.renato@spes.uniud.it}
  \and
  Luca Geatti \qquad Nicola Gigante
  \institute{Free University of Bozen-Bolzano}
  \email{\{geatti,gigante\}@inf.unibz.it}
  \and
  Angelo Montanari
  \institute{University of Udine, Italy}
  \email{angelo.montanari@uniud.it}
  \and
  Valentino Picotti
  \institute{University of Southern Denmark}
  \email{picotti@imada.sdu.dk}
}

\newif\ifdone
\donefalse

\def\titlerunning{Controller Synthesis for Timeline-based Games}
\def\authorrunning{Acampora \emph{et al.}}

\begin{document}

\maketitle

\begin{abstract}
In the timeline-based approach to planning, originally born in the space sector,
the evolution over time of a set of state variables (the timelines) is governed
by a set of temporal constraints. Traditional timeline-based planning systems
excel at the integration of planning with execution by handling \emph{temporal
uncertainty}. In order to handle general nondeterminism as well, the concept of
\emph{timeline-based games} has been recently introduced. It has been proved
that finding whether a winning strategy exists for such games is
\EXPTIME[2]-complete. However, a concrete approach to synthesize controllers
implementing such strategies is missing. This paper fills this gap, outlining 
an approach to controller synthesis for timeline-based games.
\end{abstract}


\section{Introduction}
\label{sec:introduction}

In the timeline-based approach to planning, the world is viewed as a system made
of a set of independent but interacting components whose behaviour over time
(the timelines) is governed by a set of temporal constraints, called
\emph{synchronization rules}. Timeline-based planning has been originally
introduced in the space industry~\cite{Muscettola94}, with timeline-based
planners developed and used by space agencies on both sides of the
Atlantic~\cite{CestaCFOP06,CestaCDDFOPRS07,FrankJ03,BernardiniS07,ChienRKSEMESFBST00},
both for short- to long-term mission planning~\cite{ChienRTTDNASVGA15} and
on-board autonomy~\cite{FratiniCORD11}.

While successful in practice, only recently timeline-based planning has been
studied from a theoretical perspective. The formalism has been at first compared
with traditional action-based languages \ala STRIPS, proving that they can be
expressed by means of timeline-based languages~\cite{GiganteMMO16}. Then, the
complexity of the timeline-based plan existence problem has been studied: the
problem is \EXPSPACE-complete~\cite{GiganteMMO17} over discrete time in the
general case, and \PSPACE-complete with qualitative
constraints \cite{DellaMonicaGTM20}. On dense time, the problem goes
from being \NP-complete to undecidable, depending on the syntactic restrictions
applied~\cite{BozzelliMMPW20}. The expressiveness of timeline-based languages
has also been studied from a logical perspective~\cite{DellaMonicaGMSS17}, and
an automata-theoretic point of view~\cite{DellaMonicaGMS18}.

Traditional timeline-based planning systems excel at the integration of planning
with \emph{execution} by treating explicitly the concept of \emph{temporal
uncertainty}: the exact timings of the events under control of the environment
need not to be precisely known in advance. However, general nondeterminism,
where the environment can also decide \emph{what} to do (instead of only
\emph{when} to do it) is usually not handled by these systems. To overcome this
limitation,  the concept of \emph{timeline-based games} has been recently
introduced~\cite{GiganteMOCR20}. In these games, the state variables are
partitioned between the controller and the environment, and the latter has the
freedom to play arbitrarily as long as a set of \emph{domain} rules, that define
the game arena, are satisfied. The controller plays to satisfy his set of
\emph{system} rules. A  strategy for controller is winning if it allows him/her 
to win independently from the choices of the environment.

Establishing whether a winning strategy exists for these games has been proved to be
\EXPTIME[2]-complete~\cite{GiganteMOCR20}. However, no concrete way to
synthesize a controller implementing such strategies is known. The proof technique
of the aforementioned complexity result involves the construction of
a huge (doubly exponential) \emph{concurrent game structure}, which is used to
model check some Alternating-time Temporal Logic (ATL) formulas~\cite{AlurHK02}.
While this structure is deterministic and can be in principle used as an arena
to solve a reachability game and synthesize a controller, its construction is
based on theoretical nondeterministic procedures which have no hope to
be ever concretely implemented. On the other hand, the automata-theoretic
approach by Della~Monica~\etal~\cite{DellaMonicaGMS18} provides a concrete and
effective construction of an automaton that accepts a word if and only if the
original planning problem has a solution plan. However, the automaton is
\emph{nondeterministic} and already doubly exponential, and the determinization
needed to use it as an arena would result into a further blow up and a
non-optimal procedure.

In this paper, we provide a 
concrete and computationally optimal approach
to controller synthesis for timeline-based games. We overcome the limitations
of both the above-mentioned approaches by devising a direct construction for a
\emph{deterministic} finite-state automaton that recognizes solution plans,
which is doubly exponential in size (thus not requiring the determinization of a
nondeterministic automaton). This automaton is then used as the arena of a
reachability game for which plenty of controller synthesis techniques are known
in the literature.

The paper is structured as follows. In \cref{sec:preliminaries} we introduce the
needed background on timeline-based planning and timeline-based games. Then,
\cref{sec:automaton} provides the core technical contribution of the paper,
namely the construction of the deterministic automaton recognizing solution
plans. \Cref{sec:games} uses this automaton as the game arena to solve the
controller synthesis problem. Last, \cref{sec:conclusions} summarizes the
main contributions of the work and discuss future developments.

\section{Timeline-based games}
\label{sec:preliminaries}

In this section, we introduce timeline-based games, as defined in
\cite{GiganteMOCR20}. 

\subsection{State variables, event sequences, synchronization rules}
The first basic concept is that of \emph{state variable}.

\begin{definition}[State variable]
  \label{def:timelines:state-variable}
  A \emph{state variable} is a tuple $x=(V_x,T_x,D_x,\gamma)$, where:
  \begin{itemize}
  \item $V_x$ is the \emph{finite domain} of $x$;
  \item $T_x:V_x\to2^{V_x}$ is the \emph{value transition function} of $x$,
    which maps each value $v\in V_x$ to the set of values that can immediately
    follow it; 
  \item $D_x:V_x\to\N\times\N$ is the \emph{duration function} of $x$, mapping
    each value $v\in V_x$ to a pair $(\dmin, \dmax)$ specifying respectively the
    minimum and maximum duration of any interval where $x=v$;
  \item $\gamma:V_x\to\set{\mathsf{c},\mathsf{u}}$ is the 
    \emph{controllability tag}, that, for each value $v\in V_x$, specifies whether it
    \emph{controllable} ($\gamma(v)=\mathsf{c}$) or \emph{uncontrollable}
    ($\gamma(v)=\mathsf{u}$).
  \end{itemize}
\end{definition}

Intuitively, a state variable $x$ takes a value from a finite domain and
represents a simple finite-state machine, whose transition function is $T_x$.
The behaviour over time of a set of state variables $\SV$ is defined by a 
set of \emph{timelines}, one for each variable.
Instead of reasoning about timelines directly, though, in this paper we follow 
the approach outlined in \cite{GiganteMOCR20} and represent the whole execution of a 
system, modeled by means of a set of state variables, with a single word, called \emph{event 
sequence}.

\begin{definition}[Event sequence \cite{GiganteMOCR20}]
  \label{def:event-sequence}
  Let $\SV$ be a set of state variables. Let $\actions_\SV$ be the set of all
  the terms, called \emph{actions}, of the form $\tokstart(x,v)$ or
  $\tokend(x,v)$, where $x\in\SV$ and $v\in V_x$.

  An \emph{event sequence} over $\SV$ is a sequence
  $\evseq=\seq{\event_1,\ldots,\event_n}$ of pairs $\event_i=(A_i,\delta_i)$,
  called \emph{events}, where $A_i\subseteq\actions_\SV$ is a set of actions,
  and $\delta_i\in\N_+$, such that, for any $x\in\SV$:
  \begin{enumerate}
  \item \label{def:event-sequence:start}
        for all $1\le i\le n$, if $\tokstart(x,v)\in A_i$, for some $v\in V_x$,
        then there is no $\tokstart(x,v')$ in any $\event_j$ before the
        closest $\event_k$, with $k > i$, such that $\tokend(x,v)\in A_k$ (if
        any);
  \item \label{def:event-sequence:end}
        for all $1\le i\le n$, if $\tokend(x,v)\in A_i$, for some $v\in V_x$,
        then there is no $\tokend(x,v')$ in any $\event_j$ after the
        closest $\event_k$, with $k < i$, such that $\tokstart(x,v)\in A_k$ (if
        any);
  \item \label{def:event-sequence:gaps-right}
        for all $1\le i < n$, if $\tokend(x,v)\in A_i$, for some $v\in V_x$, then
        $\tokstart(x,v')\in A_i$, for some $v'\in V_x$;
  \item \label{def:event-sequence:gaps-left}
        for all $1< i \le n$, if $\tokstart(x,v)\in A_i$, for some $v\in V_x$,
        then $\tokend(x,v')\in A_i$, for some $v'\in V_x$.
  \end{enumerate}
\end{definition}

Intuitively, an event sequence represents the evolution over time of the state
variables of the system by representing the \emph{start} and the \emph{end} of
\emph{tokens}, \ie a sequence of adjacent
intervals where a given variable takes a given
value. An event $\event_i=(A_i,\delta_i)$ consists of a set $A_i$ of actions
describing the start or the end of some tokens, happening $\delta_i$ time steps
after the previous one. In an \emph{event sequence}, events are collected to
describe a whole plan.

\Cref{def:event-sequence} intentionally implies that a started token is not
required to end before the end of the sequence, and a token can end without the
corresponding starting action to have ever appeared before. In this case we say
the event sequence is \emph{open} (on the right or on the left, respectively).
Otherwise, it is said to be \emph{closed}. An event sequence closed on the left
and open on the right is also called a \emph{partial plan}. Note that the empty
event sequence is closed on both sides for any variable. Moreover, on closed
event sequences, the first event only contains $\tokstart(x,v)$ actions and the
last event only contains $\tokend(x,v)$ actions, 
one for each variable $x$.
Given an event sequence $\evseq=\seq{\event_1,\ldots,\event_n}$ over a set of
state variables $\SV$, with $\event_i=(A_i,\delta_i)$, we define
$\delta(\evseq)=\sum_{1<i\le n}\delta_i$, that is, $\delta(\evseq)$ is the time
elapsed from the start to the end of the sequence (its duration). The amount
of time spanning a subsequence, written as $\delta_{i,j}$ when $\evseq$ is clear
from context, is then $\delta(\slice\evseq_{i,j})=\sum_{i<k\le j}\delta_k$.
Finally, given an event sequence $\evseq=\seq{\event_1,\ldots,\event_n}$, for
each $1<i\le n$, we define $\evseq_{<i}$ as $\seq{\event_1,\ldots,\event_{i-1}}$. 

In timeline-based games, the controller plays to satisfy a set of
\emph{synchronization rules}, which describe the desired behavior of the system.
Synchronization rules relate tokens, possibly belonging to different timelines,
through temporal relations among token endpoints. Let \SV be a set of state variables
and $\toknames = \set{a,b,\ldots}$ be an arbitrary set of \emph{token names}.
Moreover, let an \emph{atomic temporal relation}, or simply \emph{atom}, be an
expression of the form $\production{term} \before_{l,u} \production{term}$, 
where $l\in\N$, $u\in\N\cup\set{\infty}$, and a \emph{term} is either $\tokstart(a)$ 
or $\tokend(a)$, for some $a\in\toknames$.
A synchronisation rule \Rule takes the following form:
\begin{equation*}\label{eq:synchronisation-rules}
  \begin{array}{rcl}
    \Rule&\bydef&a_0[x_0=v_0] \implies \E_1 \lor \E_2 \lor \dots \lor \E_k\quad where\\[0.4em]
    \E_i&\bydef&\exists a_1[x_1=v_1] a_2[x_2=v_2]\ldots a_n[x_n=v_n] \suchdot \clause_i
  \end{array}
\end{equation*}
where $a_0, \ldots, a_n \in \toknames$, $x_0,\ldots,x_n \in \SV$, $v_0,\ldots,
v_n$ are such that $v_i \in V_{x_i}$, for all $0 \leq i\leq n$, and
$\clause_i$ is a conjunction of atomic temporal relations (a clause). The elements
$a_i[x_i=v_i]$ are called \emph{quantifiers} and the quantifier $a_0[v_0=x_0]$
is called the \emph{trigger}. The disjuncts in the body are called
\emph{existential statements}.

We say that a token $\tau = (x, v, d)$ \emph{satisfies} a quantifier
$a_i[x_i=v_i]$ if $x = x_i$ and $v = v_i$. The semantics of a synchronisation
rule \Rule states that for every token satisfying the trigger, at least one of
the existential statements is satisfied. Each existential statement $\E_i$
requires the existence of some tokens, satisfying the quantifiers in its prefix,
such that the clause $\clause_i$ is satisfied. When a token satisfies the
trigger of a rule, it is said to \emph{trigger} such a rule.

For space concerns, we do not provide all the details of the semantics of synchronization rules. The reader can find them in \cite{GiganteMOCR20}. Intuitively, each time there is a token that satisfies the trigger of a rule, one of its existential statements must be satisfied as well. The existential statements in turn assert the existence of other tokens that satisfy a conjunction of atoms. 

If $a$ and $b$ are two token names, then examples of atomic relations are $\tokstart(a)\before_{3,7}\tokend(b)$ and $\tokstart(a)\before_{0,+\infty}\tokstart(b)$. Intuitively, a token name $a$ refers to a specific token, that is, a pair of $\tokstart(x,v)$ and $\tokend(x,v)$ actions in an event sequence, and $\tokstart(a)$ and $\tokend(a)$ to its endpoints. Then, an atom such as $\tokstart(a)\before_{l,u}\tokend(b)$ constrains $a$ to start before the end of $b$, with the distance between the two endpoints to be comprised between the lower and upper bounds $l$ and $u$.

Examples of synchronization rules 
are the following,
where the relations $=$ and $\before$ are respectively syntactic sugar for
$\before_{0,0}$ and $\before_{0,+\infty}$:
\begin{align*}
  a[x_s=\mathsf{Comm}] \implies {}
  & \exists b[x_g=\mathsf{Available}] \suchdot
    \tokstart(b) \before \tokstart(a) \land \tokend(a) \before \tokend(b)\\
  a[x_s=\mathsf{Science}] \implies {}
  & \exists b[x_s=\mathsf{Slewing}] c[x_s=\mathsf{Earth}] d[x_s=\mathsf{Comm}]
    \suchdot {}\\
  & \tokend(a) = \tokstart(b) \land \tokend(b) = \tokstart(c) \land
    \tokend(c) = \tokstart(d)
\end{align*}
where the variables $x_s$ and $x_g$ represent respectively the state of a
spacecraft and the visibility of the communication ground station. The first
rule requires the satellite and the ground station to synchronise their
communications, so that when the satellite is transmitting the ground station
is available for reception. The second rule instructs the system to transmit
data back to Earth after every measurement session, interleaved by the required
slewing operation. A rule whose trigger is empty ($\true$), called 
\emph{triggerless rule}, can be used to state the \emph{goal} of the system.
As an example, they allow one to force the spacecraft to
perform some scientific measurement at all:
\begin{equation*}
  \true \implies\exists a[x_s=\mathsf{Science}]
\end{equation*}

Triggerless rules have a trivial universal quantification, which means they only
demand the existence of some tokens, as specified by the existential statements.
Although triggerless rules are meant to specify the goals of a planning problem,
they can be regarded as syntactic sugar on top of the syntax described above.
Indeed, triggerless rules can be translated into triggered
rules~\cite{GiganteMOCR20}, and thus we do not consider them from here onwards.

Finally, even though our focus is on timeline-based games, we conclude the section 
by formally defining \emph{timeline-based planning problems}.
\begin{definition}[Timeline-based planning problem]
  A \emph{timeline-based planning problem} is a pair $P=(\SV,\S)$, where $\SV$ is
  a set of state variables and $\S$ is a set of synchronization rules over
  $\SV$. An event sequence $\evseq$ over $\SV$ is a solution plan for $P$ if all
  the rules in $\S$ are satisfied by $\evseq$.
\end{definition}

\subsection{The game arena}

We are now ready to introduce \emph{timeline-based games}. Their definition is quite
involved, as their structure has been designed with the goal of being strictly
more general than \emph{timeline-based planning with
uncertainty}~\cite{CialdeaMayerOU16} while being able to capture its semantics
precisely. For space concerns, we keep the exposition quite terse, but the
reader can refer to \cite{GiganteMOCR20} for details.

\begin{definition}[Timeline-based game]
  \label{def:games:game}
  A \emph{timeline-based game} is a tuple $G=(\SV_C,\SV_E,\S,\D)$,
  where $\SV_C$ and $\SV_E$ are the sets of \emph{controlled}
  and \emph{external} variables, respectively, and $\S$ and $\D$ are the sets of
  \emph{system}  and \emph{domain} synchronisation rules, respectively, both 
  involving variables from $\SV_C$ and $\SV_E$.
\end{definition}

A partial plan for $G$ is a partial plan over the state variables
$\SV_C\cup\SV_E$. Let $\partialplans_G$ be the set of all possible partial plans
for $G$, simply $\partialplans$ when there is no ambiguity.

Since $\epsilon$ is a closed event sequence and $\delta(\epsilon)=0$, the
\emph{empty} partial plan $\epsilon$ is a good starting point for the game.
Players incrementally build a richer partial plan, starting from $\epsilon$, by
playing actions that specify which tokens to start and/or to end, adding an
event that extends the event sequence, or complementing the existing
last event of the sequence. We partition all the available actions into those
that are playable by either of the two players.\fitpar

\begin{definition}[Partition of player actions]
  \label{def:games:actions-partition}
  Let $\SV=\SV_C\cup\SV_E$. The set $\actions$ of available actions over $\SV$
   is partitioned into the sets $\actions_C$ of \charlie's actions and  
   $\actions_E$ of \eve's actions, which are defined as follows:\fitpar
  \begin{align}
    \actions_C = {} &
      \underbrace{%
        \set{\tokstart(x,v)\suchthat x\in\SV_C,\; v\in V_x}%
      }_{\text{start tokens on \charlie's timelines}}\;\cup\;
      \underbrace{%
        \set{\tokend(x,v)\suchthat x\in\SV,\; v\in V_x,\; \gamma_x(v)=\ctag}%
      }_{\text{end controllable tokens}}\\
    \actions_E = {} &
      \underbrace{%
        \set{\tokstart(x,v)\suchthat x\in\SV_E,\; v\in V_x}%
      }_{\text{start tokens on \eve's timelines}}\;\cup\;
      \underbrace{%
        \set{\tokend(x,v)\suchthat x\in\SV,\; v\in V_x,\; \gamma_x(v)=\utag}%
      }_{\text{end uncontrollable tokens}}
  \end{align}
\end{definition}

Hence, players can start tokens for the variables that they own, and end the
tokens that hold values that they control. Actions are combined into
\emph{moves} that can start/end multiple tokens at once.

\begin{definition}[Moves]
  \label{def:games:moves}
  A \emph{move} $\move_C$ for \charlie is a term of the form $\wait(\delta_C)$
  or $\play(A_C)$, where $\delta_C\in\N^+$ and $\emptyset\ne
  A_C\subseteq\actions_C$ is  either a set of \emph{starting} actions or 
  a set of \emph{ending} actions.

  A \emph{move} $\move_E$ for \eve is a term of the form $\play(A_E)$ or
  $\play(\delta_E,A_E)$, where $\delta_E\in\N^+$ and $A_E\subseteq\actions_E$ is
  either a set of \emph{starting} actions or a set of \emph{ending} actions.
\end{definition}

We denote by $\moves_C$ and $\moves_E$ the set of moves playable by \charlie
and $\eve$, respectively. Moves such as $\play(A_C)$ and
$\play(\delta_E,A_E)$ can play either $\tokstart(x,v)$ actions only or
$\tokend(x,v)$ actions only. A move of the former kind is called a
\emph{starting} move, while a move of the latter kind is called an \emph{ending}
move. We consider $\wait$ moves as \emph{ending} moves. Moreover, Starting and
ending moves have to be alternated during the game.

\begin{definition}[Round]
  \label{def:games:round}
  A \emph{round} $\round$ is a pair
  $(\move_C,\move_E)\in\moves_C\times\moves_E$ of moves such that:
  \begin{enumerate}
  \item \label{def:games:round:alternation}
        $\move_C$ and $\move_E$ are either both \emph{starting} or both
        \emph{ending} moves;
  \item \label{def:games:round:paring}
        either $\round=(\play(A_C),\play(A_E))$, or
        $\round=(\wait(\delta_C),\play(\delta_E,A_E))$, with
        $\delta_E\le\delta_C$;
  \end{enumerate}
\end{definition}

A \emph{starting} (\emph{ending}) round is one made of starting (ending) moves.
Note that since \charlie cannot play empty moves and $\wait$ moves are
considered ending moves, each round is unambiguously either a starting or an
ending round. Also note that since $\play(\delta_E,A_E)$ moves are played only
in rounds together with $\wait(\delta_C)$, and $\wait(\delta_C)$ is always an
ending move, then any $\play(\delta_E,A_E)$ must be an ending move. We can now
define how a round is applied to the current partial plan to obtain the new one.
The game always starts with a single \emph{starting round}.

\begin{definition}[Outcome of rounds]
  \label{def:games:round-outcome}
  Let $\evseq=\seq{\event_1,\ldots,\event_n}$ be an event sequence, with
  $\event_n=(A_n,\delta_n)$ or $\event_n=(\emptyset,0)$ if $\evseq=\epsilon$.
  Let $\round=(\move_C,\move_E)$ be a round, let $\delta_E$ and $\delta_C$ be
  the time increments of the moves, with $\delta_C=\delta_E=1$ for $\play(A)$
  moves, and let $A_E$ and $A_C$ be the set of actions of the two moves ($A_C$
  is empty if $\move_C$ is a $\wait$ move).

  The \emph{outcome} of $\round$ on $\evseq$ is the event sequence
  $\round(\evseq)$ defined as follows:
  \begin{enumerate}
  \item \label{def:games:round-outcome:starting}
        if $\rho$ is a starting round, then $\rho(\evseq)=\evseq_{< n}\event_n'$,
        where $\event_n'\nobreak=\nobreak(A_n\cup A_C\cup A_E,\delta_n)$;
  \item \label{def:games:round-outcome:ending}
        if $\rho$ is an ending round, then $\rho(\evseq)=\evseq\event'$, where
        $\event'=(A_C\cup A_E,\delta_E)$;
  \end{enumerate}
  We say that $\round$ is \emph{applicable} to $\evseq$ if:
  \begin{enumerate}[label=\alph*)]
  \item \label{def:games:round-outcome:integrity}
        the above construction is well-defined, \ie $\round(\evseq)$ is a
        valid event sequence by \cref{def:event-sequence};
  \item \label{def:games:round-outcome:alternation}
        $\round$ is an ending round if and only if $\evseq$ is open for all
        variables.
  \end{enumerate}
\end{definition}

We say that a single move by either player is applicable to $\evseq$ if there is
a move for the other player such that the resulting round is applicable to
$\evseq$.

The game starts from the empty partial plan $\epsilon$, and players play in
turn, composing a round from the move of each one, which is applied to the
current partial plan to obtain the new one.

It is now time to define the notion of \emph{strategy} for each player, and of
\emph{winning strategy} for \charlie.

\begin{definition}[Strategies]
  \label{def:games:strategies}
  A \emph{strategy for Charlie} is a function
  $\strategy_C:\partialplans\to\moves_C$ that maps any given partial plan
  $\evseq$ to a move $\move_C$ applicable to $\evseq$.
  A \emph{strategy for Eve} is a function
  $\strategy_E:\partialplans\times\moves_C\to\moves_E$ that maps a partial
  plan $\evseq$ and a move $\move_C\in\moves_C$ applicable to $\evseq$, to a
  $\move_E$ such that $\round=(\move_C,\move_E)$ is applicable to
  $\evseq$.
\end{definition}

A sequence $\rounds=\seq{\round_0,\ldots,\round_n}$ of rounds is called a
\emph{play} of the game. A play is said to be \emph{played according to} some
strategy $\strategy_C$ for \charlie, if, starting from the initial partial plan
$\evseq_0=\epsilon$, it holds that $\round_i=(\strategy_C(\Pi_{i-1}),
\move_E^i)$, for some $\move_E^i$, for all $0<i\le n$, and to be played
according to some strategy $\strategy_E$ for \eve if $\round_i=(\move_C^i,
\strategy_E(\Pi_{i-1},\move_C^i))$, for all $0<i\le n$. It can be seen that for
any pair of strategies $(\strategy_C,\strategy_E)$ and any $n\ge0$, there is a
unique run $\rounds_n(\strategy_C,\strategy_E)$  of length $n$ played according
both to $\strategy_C$ and $\strategy_E$.

Then, we say that a partial plan $\evseq$, and the play $\rounds$ such that
$\evseq=\rounds(\epsilon)$, are \emph{admissible}, if the partial plan satisfies
the domain rules, and are \emph{successful} if the partial plan satisfies the
system rules.

\begin{definition}[Admissible strategy for \eve]
  \label{def:games:admissible-strategy}
  A strategy $\strategy_E$ for \eve is \emph{admissible} if for each strategy
  $\strategy_C$ for \charlie, there is $k\ge 0$ such that the play
  $\rounds_k(\strategy_C,\strategy_E)$ is admissible.
\end{definition}

\charlie wins if, \emph{assuming} domain rules are respected, he manages to
satisfy the system rules no matter how \eve plays.

\begin{definition}[Winning strategy for \charlie]
  \label{def:games:winning-strategy}
  Let $\strategy_C$ be a strategy for \charlie. We say that $\strategy_C$ is a
  \emph{winning strategy} for \charlie if for any \emph{admissible} strategy
  $\strategy_E$ for \eve, there exists $n\ge0$ such that the play
  $\rounds_n(\strategy_C,\strategy_E)$ is successful.
\end{definition}

We say that \charlie \emph{wins} the game $G$ if he has a winning strategy,
while \eve \emph{wins} the game if a winning strategy for \charlie does not
exist.

\section{A deterministic automaton for timeline-based planning}
\label{sec:automaton}

In this section we encode a timeline-based planning problem into a
\emph{deterministic} finite state automaton (DFA) that recognises all and only
those event sequences that represent solution plans for such problem. This
automaton will form the basis for the game arena solved in the next section. The
words accepted by the automaton are \emph{event sequences} representing solution
plans.

Let $P=(\SV, S)$ be a timeline-based planning problem. To get a finite alphabet,
we define $d=\max(L,U)+1$, where $L$ and $U$ are in turn the \emph{maximum}
lower and (finite) upper bounds appearing in any rule of $P$, and we account
only for event sequences such that the distance between two consecutive events
is at most $d$. It can be easily seen that this assumption does not loose
generality (for a proof, see Lemma~4.8 in \cite{Gigante19}). Hence, the symbols
of the alphabet $\Sigma$ are \emph{events} of the form $\event = \pair{A,
\delta}$, where $A \subseteq \actions_\SV$ and $1\le\delta\le d$. Formally,
$\Sigma=2^{\actions_\SV}\times\ar{d}$, where $\ar{d}=\set{1,\ldots,d}$. Note
that the size of $\Sigma$ is exponential in the size of the problem. Moreover,
we define the amount $\window(P)$ as the product of all the non-zero
coefficients appearing as upper bounds in rules of $P$. Intuitively, $\window(P)$
is the maximum amount of time a rule of $P$ can \emph{count} far away from the
occurrence of the quantified tokens. For example, consider the following rule:
\begin{align*}
  a_0[x_0=v_0]\to{} &\exists a_1[x_1=v_1] a_2[x_2=v_2] a_3[x_3=v_3] \suchdot \\
  &\tokstart(a_1)\before_{[4,14]}\tokend(a_0)
  \land \tokend(a_0)\before_{[0,+\infty]}\tokend(a_2) \land \tokstart(a_2)\before_{[0,3]}\tokend(a_3) 
\end{align*}
In this case, $\window(P)$ (assuming this is the only rule of the problem),
would be $3\cdot 14=42$. This means the rule can precisely account for what
happens at most $42$ time point from the occurrence of its quantified tokens.
For example, if the token $a_1$ appears at a given distance from $a_0$, it has
to be at less than $42$ time points (less than $14$, in particular), and any
modification of the plan that changes such distance has the potential to break
the satisfaction of the rule. Instead, what happens further away from $a_0$ only
affects the satisfaction of the rule \emph{qualitatively}. Suppose the tokens
$a_2$ and $a_3$ lie at $100$ time points from $a_0$ (at most $3$ time steps from
each other). Changing this distance (while maintaining the qualitative order
between tokens) cannot ever break the satisfaction of the rule. See
\cite{Gigante19} for a precise account of the properties of $\window(P)$.

A key observation underlying our construction is that every atomic temporal
relation $T \before_{l,u} T'$ can be rewritten as the conjunction of two
inequalities $T^\prime - T \leq u$ and $T - T^\prime \leq -l$, and that the
clause \clause of an existential statement \E can be rewritten as a system of
difference constraints $\nu(\clause)$ of the form $T - T' \leq n$, with $n \in
\Z_{+\infty}$. Then, the system $\nu(\clause)$ can be conveniently represented
by a squared matrix $D$ indexed by terms, where the entry associated with $D[T,
T']$ gives the upper bound on $T- T'$. Such matrices, which take the name of
\emph{Difference Bound Matrices} (DBMs)~\cite{dill1989timing,peron2007abstract},
can be conveniently updated as the plan evolves to keep track of the
satisfaction of the atomic temporal relations among terms. In building a DBM for
the system of constraints $\nu(\clause)$, we augment the system with constraints
of kind $\tokstart(a_i)-\tokend(a_i)\leq-\dmin^{x_i=v_i}$ and
$\tokend(a_i)-\tokstart(a_i)\leq\dmax^{x_i=v_i}$, for any quantified token
$a_i[x_i=v_i]$ of $\E$. Moreover, if two different bounds $T - T'\le n$ and
$T-T'\le n'$ with $n'<n$ belong to $\nu(\clause)$, we keep only $T-T'\le n'$. As
an example, the DBM for the existential graph of the rule above is the one in
\cref{fig:dbm}.
\begin{figure}
  \begin{equation*}
    \begin{array}{rcccccccc}\toprule
      &\tokstart(a_0) & \tokend(a_0) & \tokstart(a_1) & \tokend(a_1) &
      \tokstart(a_2) & \tokend(a_2) & \tokstart(a_3) & \tokend(a_3) \\\midrule
      \tokstart(a_0) \\
      \tokend(a_0) & & & -4 & & & \\
      \tokstart(a_1) & & 14 \\
      \tokend(a_1) \\
      \tokstart(a_2)  & & & & & & & & 3 \\
      \tokend(a_2) & & 0 \\
      \tokstart(a_3) \\
      \tokend(a_3)  & & & & & 0 & & & \\\bottomrule
    \end{array}
  \end{equation*}
  \caption{DBM of an example synchronization rule. Missing entries are intended to be $+\infty$.}
  \label{fig:dbm}
\end{figure}
Note that, when the bounds of the temporal relations are translated into a DBM,
there is no longer a distinction between \emph{lower} and \emph{upper} bounds.
However, for some of the entries we can retrieve their original meaning. Indeed,
if $D[T,T'] < 0$, then such entry is the lower bound of a temporal relation
$T\before_{l,u}T'$, whereas, if $D[T, T'] > 0$, it is the upper bound of a
relation $T' \before_{l,u} T$.

On top of DBMs, we define the concept of \emph{matching structure}, a data
structure that allows us to manipulate and reason about partially matched
existential statements, \ie existential statements of which only a part of the
requests has already been satisfied by the part of the word already read, while
the rest can be still potentially matched in the future.

\begin{definition}[Matching Structure]
  \label{def:matching-structure}
  Let $\E\equiv \exists a_1[x_1 = v_1] \dots a_m[x_m = v_m] \,.\, \clause$ be
  the existential statement of a synchronisation rule $\Rule \equiv a_0[x_0 =
  v_0] \rightarrow \E_1 \lor \dots \lor \E_k$ over the set of state variables
  \SV.

  The \emph{matching structure} for $\E$ is a tuple $\M_{\E} = (V, D, M, t)$
  where:
  \begin{itemize}
  \item $V$ is the set of terms $\tokstart(a)$ and $\tokend(a)$ for
    $a\in\set{a_0, \dots, a_m}$;
  \item $D \in \Z_{+\infty}^{|V|^2}$ is a DBM indexed by terms of $V$ where
    $D[T,T']=n$ if $(T-T'\le n) \in \nu(\clause)$, $D[T,T']=0$ if $T=T'$, and $D[T,T']=+\infty$ otherwise;
  \item $M \subseteq V$ and $0\le t \le \window(P)$.
  \end{itemize}
\end{definition}

The set $M$ contains the terms of $V$ that the matching structure has correctly
matched over the event sequence read so far. With $\overline{M} = V \setminus M$
we denote the actions that we have yet to see. Then, we say that a matching
structure $\M$ is \emph{closed} if $M = V$, it is \emph{initial} if $M =
\emptyset$ and it is \emph{active} if it is not closed and $\tokstart(a_0) \in
M$. Intuitively, a matching structure is \emph{active} if its trigger has been
matched over the word the automaton is reading. Then, when all the terms have
been matched over the word, the matching structure becomes \emph{closed}. The
component $t$ is the time elapsed since $\tokstart(a_0)$ has been matched.
When time flows, a matching structure can then be updated as follows.








\begin{definition}[Time shifting]
  \label{def:time-shift}
  Let $\delta > 0$ be a positive amount of time, and $\M = (V, D, M, t)$ be a
  matching structure. The result of shifting $\M$ by $\delta$ time units,
  written $\M + \delta$, is the matching structure $\M^\prime = (V, D^\prime, M,
  t')$ where:
  \begin{itemize}
  \item for all $T, T' \in V$:
    \begin{equation*}
      D^\prime[T,T'] =
      \begin{cases}
        D[T,T'] + \delta &\text{if } T \in M \text{ and } T' \in
        \overline{M}\\
        D[T,T'] - \delta &\text{if } T \in \overline{M} \text{ and } T' \in
        M\\
        D[T,T'] &\text{otherwise}
      \end{cases}
    \end{equation*}
  \item and
    \[
      t' =
      \begin{cases}
        t+\delta & \text{if } \M \text{ is \emph{active}}\\
        t & \text{otherwise}
      \end{cases}
    \]
  \end{itemize}
\end{definition}

\begin{definition}[Matching]
  \label{def:matching}
  Let $\M = (V, D, M, t)$ be a matching structure and $I \subseteq \overline{M}$
  a set of matched terms. A matching structure $\M^\prime = (V, D, M^\prime, t)$
  is the result of matching the set $I$, written $\M \cup I$, if $M^\prime = M
  \cup I$.
\end{definition}

Beside updating the reference $t$ to the trigger occurrence of an active
matching structure, \Cref{def:time-shift} dictates how to update the entries of
the DBM. In particular, the distance bounds between any pair of terms $T$ and
$T'$ where one is in $M$ and the other is not are tighten by the elapsing of
time: when $T\in M$ and $T'\in\overline{M}$, $D[T,T']$ is a lower bound loosen
by adding the elapsed time $\delta$, when $T\in\overline{M}$ and $T'\in M$,
$D[T,T']$ is an upper bound tighten by subtracting $\delta$. For example,
consider the DBM in \cref{fig:dbm} and consider the pair of terms
$\tokstart(a_1)$ and $\tokend(a_0)$. $D[\tokstart(a_1),\tokend(a_0)]=-4$,
meaning that $\tokend(a_0)-\tokstart(a_1)\le 4$ must hold. Suppose
$\tokstart(a_1)\in M$ (\ie it has been matched), and $\tokend(a_0)\not\in M$ (it
still has to). Now, if $1$ time point passes, the entry in the DBM is
incremented and updated to $-4+1=-3$, which corresponds to the constraint
$\tokend(a_0)-\tokstart(a_1)\le 3$. This reflects the fact that to be able to
satisfy the constraint, $\tokend(a_0)$ has now only $3$ time steps left before
it is too late.
\Cref{def:matching} tells us how to update the set $M$ of a matching structure.

To correctly match an existential statement while reading an event sequence, a
matching structure is updated only as long as no violations of temporal
constraints are witnessed. As such, an event is classified from the standpoint
of a matching structure as \emph{admissible} or not.

\begin{definition}[Admissible Event]
  An event $\event = (A, \delta)$ is \emph{admissible} for a matching structure
  $\M_{\E} = (V, D, M, t)$ if and only if, for every $T \in M$
  and $T' \in \overline{M}$, $\delta \leq D[T',T]$, \ie the elapsing of
  $\delta$ time units does not exceed the upper bound of some term $T'$ not yet
  matched by $\M_{\E}$.
\end{definition}

Each admissible event $\event$ read from the word can be matched with a subset
of the terms of the matching structure. There are usually more than one way to
match events and terms. The following definition makes this choice explicit.

\begin{definition}[$I$-match Event]\label{def:match-event}
  Let $\M_{\E} = (V, D, M, t)$ be a matching structure and  $I \subseteq
  \overline{M}$. An $I$\emph{-match event} is an admissible event $\event = (A,
  \delta)$ for $\M_{\E}$ such that:
  \begin{enumerate}
  \item for all token names $a \in \mathsf{N}$ quantified as $a[x = v]$ in $\E$
    we have that:\label{def:match-event:good-match}
    \begin{enumerate}
    \item if $\tokstart(a) \in I$, then $\tokstart(x, v) \in A$;
      \label{def:match-event:good-match:start}
    \item $\tokend(a) \in I$ if and only if $\tokstart(a) \in M$ and $\tokend(x,v) \in
      A$;\label{def:match-event:good-match:end}
    \end{enumerate}
  \item and for all $T \in I$ it holds that:\label{def:match-event:relations}
    \begin{enumerate}
    \item \label{def:match-event:preceding-terms} for every other term $T' \in
      V$, if $D[T',T] \leq 0$, then $T' \in M \cup I$;
    \item \label{def:match-event:lower-bounds} for all $T' \in M$, $\delta \geq
      -D[T',T]$, \ie all the lower bounds on $T$ are satisfied;
    \item \label{def:match-event:zero-no-bounds} for each other term $T' \in I$,
      either $D[T',T] = 0$, $D[T,T'] = 0$, or $D[T',T] = D[T, T'] = +\infty$.
    \end{enumerate}
  \end{enumerate}
\end{definition}

Intuitively, an event is an $I$-match event if the actions in the event
correctly match the terms in $I$. \Cref{def:match-event:good-match} ensures that
each term is correctly matched over an action it represents and, most
importantly, that the endpoints of a quantified token correctly identify the
endpoints of a token in the event sequence. \Cref{def:match-event:relations}
ensures that matching the terms in $I$ does not violate any atomic temporal
relation. In particular, \Cref{def:match-event:preceding-terms} deals with the
qualitative aspect of an ``happens before'' relation, while \Cref{%
  def:match-event:lower-bounds,%
  def:match-event:zero-no-bounds%
} deal with the quantitative aspects of the lower bounds of these relations.
Note that an $\emptyset$-event is admitted.

Let $\matchstructs_P$ be the set of all the matching structures for a planning
problem $P$. By \Cref{def:match-event}, a single event can represent several
$I$-match events for a matching structure, hence a matching structure can evolve
into several matching structures, one for each $I$-match event. Such
evolution is defined as a ternary relation $S \subseteq\matchstructs_P \times
\Sigma \times \matchstructs_P$ such that $(\M, (A, \delta), \M^\prime) \in {S}$
if and only if $(A, \delta)$ is an $I$-match event for $\M$ and $\M^\prime = (\M
+ \delta) \cup I$. To deal with the nondeterministic nature of this relation,
states of the automaton will comprise sets of matching structures collecting all
the possible outcomes of $S$, so that suitable notation for working with sets of
matching structures, denoted by $\Upsilon$ hereafter, is introduced. We define
$\Upsilon^\Rule_t\subseteq\Upsilon$ as the set of all the \emph{active} matching
structures $\M\in\Upsilon$ with timestamp $t$, associated with any existential
statement of $\Rule$. Intuitively, matching structures in $\Upsilon^\Rule_t$
contribute to the fulfilment of the same triggering event for the rule \Rule
(because they have the same timestamp), regardless of the existential statement
they represent. We also define $\Upsilon_\bot\subseteq\Upsilon$ as the set of
\emph{non active} matching structures of $\Upsilon$. A set $\Upsilon$ is
\emph{closed} if there exists $\M \in \Upsilon$ such that $\M$ is \emph{closed}.
Lastly, a function $\step_\event$ extends the relation $S$ to sets of matching
structures: $\step_\event(\Upsilon) = \set{\M' | (\M,\event,\M^\prime)\in S,
\text{ for some } \M \in \Upsilon}$.

We are now ready to define the automaton. If $\E$ is an existential statement,
let $\mathbb{E}_\E$ be the set of all the existential statements of the same
rule of $\E$. Let $\mathbb{F}_P$ be the set of functions mapping each existential
statement of $P$ to a set of existential statements, and let $\mathbb{D}_P$ be the
set of functions mapping each existential statement to a set of matching
structures. A simple automaton $\TV_P$ that checks the transition function and
duration functions of the variables is easy to define. Then, given a
timeline-based planning problem $P=(\SV,S)$, the corresponding automaton is
$A_P=\TV_P\cap\S_P$ where $\S_P$, the automaton that checks the satisfaction of
the synchronization rules, is defined as $\S_P = (Q, \Sigma, q_0, F, \tau)$,
where:
\begin{enumerate}
\item $Q = 2^\matchstructs \times \mathbb{D} \times \F \cup \set{\bot}$ is the
  finite set of states, \ie states are tuples of the form $\langle \Upsilon,
  \Delta, \Phi \rangle\in2^\matchstructs \times \mathbb{D} \times \F$, plus a
  sink state $\bot$;
\item $\Sigma$ is the input alphabet defined above;
\item the initial state $q_0 = \langle \Upsilon_0, \Delta_0, \Phi_0 \rangle$ is
  such that $\Upsilon_0$ is the set of initial matching structures of the
  existential statements of $P$ and, for all existential statements $\E$ of $P$,
  we have $\Delta_0(\E) = \emptyset$ and $\Phi_0(\E) = \mathbb{E}_\E$;
\item $F \subseteq Q$ is the set of final states defined as:
  \[
    F = \Set{ \langle \Upsilon, \Delta, \Phi \rangle \in Q |
      \begin{gathered}
        \M \text{ is not \emph{active} for all } \M \in
        \Upsilon\\
        \text{and }\Delta(\E)=\emptyset\text{ for all }\E\text{ of } P
      \end{gathered}}
  \]
\item $\tau : Q \times \Sigma \rightarrow Q$ is the transition function that
  given a state $q=\langle \Upsilon, \Delta, \Phi \rangle$ and a symbol $\event
  = (A, \delta)$ computes the new state $\tau(q,\event)$. Let
  $\step^\E_\event(\Upsilon^\Rule_t)=\set{\M_\E \mid
  \M_\E\in\step_\event(\Upsilon^\Rule_t)}$. Moreover, let $\Psi^\Rule_t = \set{ \E |
  \M_{\E} \in \step_\event(\Upsilon^\Rule_t)}$. Then, the updated components of
  the state are based on what follows, where $W = \window(P)$:
  \begin{align*}
    \Upsilon' &= \step_\event(\Upsilon_\bot) \cup \bigcup \Set{
      \step_\event(\Upsilon^\Rule_t) |
      \text{$t<W-\delta$ and 
      $\step_\event(\Upsilon^\Rule_t)$ is not \emph{closed}}} \\
    \Delta'(\E) &=\begin{cases}
        \step^\E_\event(\Upsilon^\Rule_t) & \text{where $t$ is the minimum such that $t\ge W-\delta$ and $\step^\E_\event(\Upsilon^\Rule_t)\ne\emptyset$} \\
        \step_\event(\Delta(\E)) & \text{if such $t$ does not exist}
      \end{cases}\\
    \Phi'(\E) &= \begin{cases}
      \mathbb{E}_\E\quad\text{if $\E\in\Psi(\E')$ for some $\E'$ such that 
      $\Delta'(\E')$ is \emph{closed}}  \\
      \Phi(\E) \setminus 
        \set{
          \E'\mid \exists t\ge W-\delta \suchdot \E'\in\Psi^\Rule_t 
          \land \E\not\in\Psi^\Rule_t
        } \quad \text{otherwise}
    \end{cases}
  \end{align*}

  Let $\Delta''(\E)=\Delta'(\E)$ unless there is an $\E'$ with $\E\in\Phi'(\E')$
  such that $\Delta'(\E')$ is \emph{closed}, in which case
  $\Delta''(\E)=\emptyset$. Then, $\tau(q,\event)=\seq{\Upsilon', \Delta'',
  \Phi'}$ if the following holds:
  \begin{enumerate}
  \item for every $\Upsilon^\Rule_t$, $\step_\event(\Upsilon^\Rule_t) \neq
    \emptyset$, and \label{dfa:delta:no-failed-step}
  \item for every synchronisation rule $\Rule \equiv a_0[x_0=v_0] \rightarrow
    \E_1 \lor \dots \lor \E_n$ in $S$, if $\tokstart(x_0, v_0) \in A$, then
    there exists $\M_{\E_i} = (V,D,M,0) \in \Upsilon'$, 
    with $i \in \{1\dots n\}$, such that $\tokstart(a_0) \in M$;\label{dfa:delta:trigger-capture}
  \end{enumerate}
  Otherwise, $\tau(q,\event)=\bot$.
\end{enumerate}

Let us explain what is going on. The first component $\Upsilon$ of an automaton
state $q=(\Upsilon, \Delta,\Phi)$ is a set of matching structures that keeps
track of what have been tracked so far. Intuitively, the automaton precisely
keeps track of what happened to the last $\window(P)$ time points, and only
summarizes what happened before that window, which is what allows us to keep the
size under control. Any matching structure in $\Upsilon$ has $t<\window(P)$.
Matching structures in $\Upsilon$ evolve following the $\step$ function, until
they are closed or the $t$ component reaches $\window(P)$. Matching structures
that reach $\window(P)$ are promoted to a new role. Their new task is to record
the pieces of existential statements that still have to be matched in order to
satisfy all the trigger events of $\Rule$ that no longer fit into the
\emph{recent history} of the event sequence (\ie the last $\window(P)$ time
points). These matching structures are not stored in $\Upsilon$ though, they are
summarized by the function $\Delta$ that maps each existential statement $\E$ of
a rule $\Rule$ to the set of matching structures for $\E$ with $t=\window(P)$.

When a set $\Upsilon^\Rule_t$ exceeds the bound $\window(P)$, the $\Delta$
function must be updated by merging the information of $\Upsilon^\Rule_t$ to the
information already present in $\Delta$. Now, it has to be noted that, by
closing a set $\Delta(\E)$, we can not conclude that every event that triggered
$\Rule$ actually satisfies $\Rule$. Indeed, there can be sets $\Delta(\E)$ and
$\Delta(\E')$ that are in charge of the satisfaction of the same rule $\Rule$,
but for different trigger events, and closing $\Delta(\E)$ does not imply that
$\Rule$ has been satisfied. The opposite case may also arise, in which
$\Delta(\E)$~and~$\Delta(\E')$ contribute to the fulfilment of the same trigger
events and closing either set suffices to satisfy $\Rule$. To overcome the
information lost when a set of matching structures gets added to the $\Delta$
function, the $\Phi$ function (the third component of the automaton states) maps
each existential statement $\E$ to the set of existential statements $\E'$ such
that $\Delta(\E')$ tracks the fulfilment of the same trigger events of the set
$\Delta(\E)$. We use $\Phi$ as follows: when a set $\Delta(\E)$ gets closed, we
can discard its matching structures and all the matching structures of the sets
$\Delta(\E')$, with $\E' \in \Phi(\E)$.

One can prove the soundness and completeness of our construction.

\begin{theorem}(Soundness and completeness)
  \label{thm:soundness-completeness}
  Let $P=(\SV, S)$ be a timeline-based planning problem and let $\A_P$ be the
  associated automaton. Then, any event sequence $\evseq$ is a solution plan for
  $P$ if and only if $\evseq$ is accepted by $\A_P$.
\end{theorem}

Recall that we assumed the timestamp of each event of event sequences to be
bounded, but since events can have an empty set of actions,
\cref{thm:soundness-completeness} can actually deal with arbitrary event
sequences, after adding suitable empty events. Now, let us look at the size of
the automaton. Let $E$ be the overall number of existential statements in $P$,
which is linear in the size of $P$. It can be seen that $\abs{\mathbb{D}_P} \in
\O({(2^{\abs{\matchstructs_P}})}^E)= \O(2^{E\cdot\abs{\matchstructs_P}})$, \ie
the number of $\Delta$ functions is doubly exponential in the~size~of~$P$. Then,
observe that $\lvert\mathbb{F}_P\rvert \in \mathcal{O}({(2^E)}^E) =
\mathcal{O}(2^{E^2})$. Then, $\abs{\S_P} \in \O(\abs{\Sigma}\cdot
2^{\abs{\matchstructs_P}})$, that is, the size of $\S_P$ is at most exponential
in the number of possible matching structures. To bound this number, let $N$ be
the largest finite constant appearing in $P$ as bound in any atom or value
duration function and let $L$ be the length of the largest existential prefix of
an existential statement occurring inside a rule of $P$. Notice that $N$ is
exponential in the size of $P$, since constants are expressed in binary, while
$L \in \O(\abs{P})$. Then, the entries of a DBM for $P$, of which there is a
number quadratic in $L$, are constrained to take values within the interval
$\ar{-N,N}$ (excluding the infinitary value $+\infty$), whose size is linear in
$N$. By \Cref{def:matching-structure}, it follows that, for the planning problem
$P$, $\abs{\matchstructs_P} \in \O(N^{L^2} \cdot 2^L \cdot \window(P))$, \ie the
number of matching structures is at most exponential in the size of $P$. Hence,
we proved the following:
\begin{theorem}[Size of the automaton]
  Let $P=(\SV, S)$ be a timeline-based planning problem and let $\A_P$ be the
  associated automaton. Then, the size of $A_P$ is at most doubly-exponential in
  the size of $P$.
\end{theorem}
Note that this is the same size as the automaton built by
Della~Monica~\etal~\cite{DellaMonicaGMS18}, but their automaton was
\emph{nondeterministic}, while ours is by construction deterministic, essential 
for its use as a game arena.



\section{Controller synthesis}
\label{sec:games}

In this section we use the deterministic automaton constructed above to
obtain a deterministic arena where we can solve a simple reachability game
for checking the existence of (and, in this case, to synthesize)
a controller for the corresponding timeline-based game.

\subsection{From the automaton to the arena}

Let $G=(\SV_C, \SV_E, \S, \D)$ be a timeline-based game. We use the construction
of the automaton explained in the previous section in order to obtain a game
arena. However, the automaton construction considers a planning problem with a
single set of synchronization rules, while here we have to account for the roles
of both $\S$ and $\D$. 

To do that, let $A_\S$ and $A_\D$ be the deterministic automata built over the
timeline-based planning problem $P_\S=(\SV_C\cup\SV_E, \S)$ and
$P_\D=(\SV_C\cup\SV_E, \D)$, respectively. We define the automaton $A_G$ as
$\overline{A_\D}\cup A_\S$, \ie the union of $A_\S$ with the complement of
$A_\D$. Note that these are all standard automata-theoretic constructions over
DFAs. Any accepting run of $A_G$ represents either a plan that violates the
domain rules or a plan that satisfies both the domain and the system rules, in
conformance with \cref{def:games:winning-strategy}. Note that $A_G$ is
deterministic and can be built from $A_\D$ and $A_\S$ with only a polynomial
increase in size.

Now, the $A_G$ automaton is still not suitable as a game arena, because the
moves of the timeline-based game are not directly visible in the labels of the
transitions. In other words, the $A_G$ automaton reads events, while we need an
automaton that reads game \emph{moves}. In particular, a single transition in
the automaton can correspond to different combinations of rounds, since the
presence of $\wait(\delta)$ moves is not explicit in the transition. For
example, an event $\event=(A, 5)$ can be the result of a $\wait(5)$ move by
\charlie followed by a $\play(5, A)$ move by $\eve$, or by any $\wait(\delta)$
move with $\delta>5$ followed by $\play(5, A)$. Hence, we need to further adapt
$A_G$ to obtain a suitable arena.

Let $A_G=(Q,\Sigma, q_0, F, \tau)$ be the automaton built as described before.
Let $\event=(A,\delta)$ be an event. If $\delta>1$, this transition must have
resulted from \charlie playing a $\wait(\delta')$ move with $\delta'\ge\delta$.
However, if $A$ contains any $\tokend(x,v)$ action with $x\in\SV_C$, this is for
sure the result of more than one pair of starting/closing rounds. In order to
simplify the construction below, we remove this possibility beforehand. More
formally, we define a slightly different automaton $A_G'=(Q,\Sigma,q_0,F,\tau')$
where $\tau'$ is now a \emph{partial} transition function (\ie the automaton
becomes \emph{incomplete}) that agrees with $\tau$ on everything excepting that
transitions $\tau(q,(\actions,\delta))$ is \emph{undefined} if $\delta>1$ and
$\actions$ contains any $\tokend(x,v)$ action with $x\in\SV_C$. You can see an
example of this operation in \cref{fig:constructions}, on the left. Note that
this removal does not change the plans accepted by the automaton because for
each transition $\tau(q,(\actions,\delta))=q'$ with $\delta > 1$ there are two
transitions $\tau(q,(\emptyset,\delta-1))=q''$ and $\tau(q'',(\actions,1))=q'$.

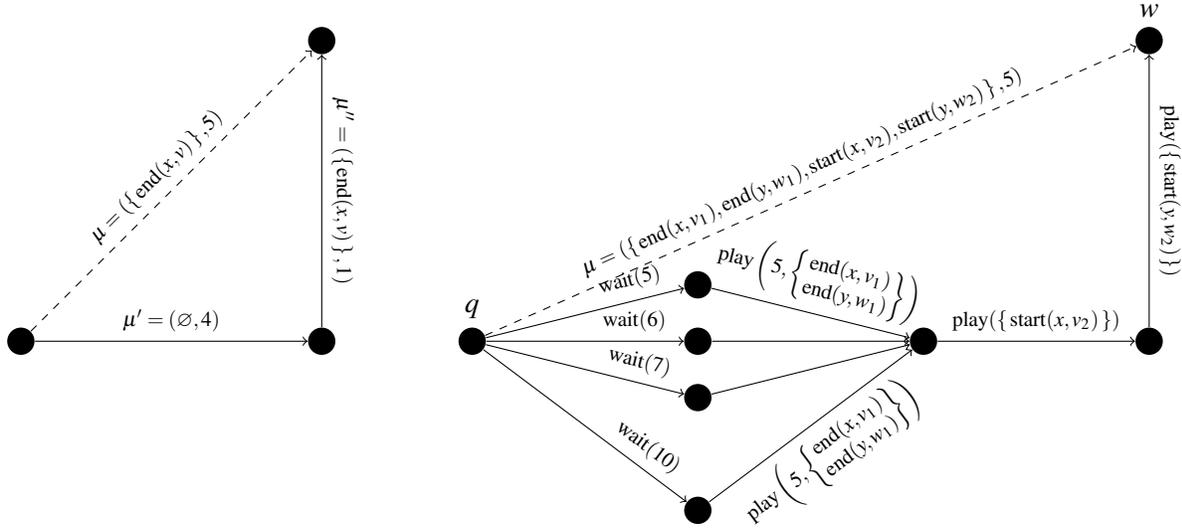
\begin{figure}
  \begin{tikzpicture}[state/.style={fill, circle, minimum width=5pt}]
    \path (0,0) node[state] (n1) { }
          (4,4) node[state] (n2) { }
          (4,0) node[state] (n3) { };

    \path[draw,dashed,->] (n1) -- (n2) node[midway, above, sloped, font=\scriptsize] { 
      $\event=(\set{\tokend(x,v)}, 5)$
    };

    \path[draw, ->] (n1) -- (n3) node[midway, above, font=\scriptsize] { 
      $\event'=(\emptyset, 4)$
    };
    \path[draw, ->] (n3) -- (n2) node[midway, above, sloped, rotate=180, font=\scriptsize] { 
      $\event''=(\set{\tokend(x,v)}, 1)$
    };

    \begin{scope}[xshift=6cm]
      \path (0,0) node[state] (n1) { } node[above, outer sep=5pt] {$q$}
            (9,4) node[state] (n2) { } node[above, outer sep=5pt] {$w$};

      \path[draw,dashed,->] (n1) -- (n2) node[midway, above, sloped, font=\scriptsize] { 
        $\event=(\set{\tokend(x,v_1),\tokend(y,w_1),\tokstart(x,v_2),\tokstart(y,w_2)}, 5)$
      };

      \path (3,0) node[state] (q6) { } node[above right] { }
            (3,0.75) node[state] (q5) { } node[above right] { }
            (3,-0.75) node[state] (q7) { } node[above right] { }
            (3,-2.25) node[state] (q10) { } node[above right] { };

      \path[dashed] (q7) -- (q10);

      \path[draw,->] (n1) -- (q5) node[above, sloped, near end, font=\scriptsize] { $\wait(5)$ };
      \path[draw,->] (n1) -- (q6) node[above, sloped, near end, font=\scriptsize] { $\wait(6)$ };
      \path[draw,->] (n1) -- (q7) node[above, sloped, near end, font=\scriptsize] { $\wait(7)$ };
      \path[draw,->] (n1) -- (q10) node[above, sloped, near end, font=\scriptsize] { $\wait(10)$ };

      \path (6,0) node[state] (q'') { };
      \path[draw,->] (q5) -- (q'') node[above, sloped, midway, font=\scriptsize] { $\play\left(5, 
        \left\{\begin{array}{@{}c@{}}
          \tokend(x,v_1)\\
          \tokend(y,w_1)
        \end{array}\right\}\right)$ };
      \path[draw,->] (q6) -- (q'');
      \path[draw,->] (q7) -- (q'');
      \path[draw,->] (q10) -- (q'') node[below, sloped, midway, font=\scriptsize] { $\play\left(5, 
      \left\{\begin{array}{@{}c@{}}
        \tokend(x,v_1)\\
        \tokend(y,w_1)
      \end{array}\right\}\right)$ };

      \path (9,0) node[state] (q''') { };
      \path[draw,->] (q'') -- (q''') node[midway, above,font=\scriptsize] {
        $\play(\set{\tokstart(x,v_2)})$
      };
      \path[draw,->] (q''') -- (n2) node[midway, sloped,rotate=180, above,font=\scriptsize] {
        $\play(\set{\tokstart(y,w_2)})$
      };
    \end{scope}
  \end{tikzpicture}
  \caption{On the left, the removal of transitions $\event=(A,\delta)$ with
  $\delta>1$ and ending actions of controllable tokens in $A$. On the right, the
  transformation of a transition of the $A_G$ into a sequence of transitions in
  $A^a_G$, with $x\in\SV_C$, $y\in\SV_E$, and
  $\gamma_x(v_1)=\gamma_y(w_1)=\mathsf{u}$.}
  \label{fig:constructions}
\end{figure}

Now we can transform the automaton in order to make the game rounds, and
especially $\wait(\delta)$ moves, explicit. Intuively, each transition of the
automaton is split into four transitions explicitating the four moves of the two
rounds. Given the automaton $A_G'=(Q,\Sigma, q_0, F, \tau')$, we define the
automaton $A_G^a=(Q^a,\Sigma^a, q_0^a, F^a, \tau^a)$, which will be the arena of
our game, as follows:
\begin{enumerate}
  \item $Q^a=Q\cup\set{q_\delta\mid 1\le\delta\le d}\cup\set{q_{\delta,A}\mid 1\le\delta\le d, A\subseteq \mathsf{A}}$ is the set of states;
  \item $\Sigma^a=\moves_C\cup\moves_E$, \ie the alphabet is turned into the set
    of moves of the two players;
  \item $q_0^a=q_0$ and $F^a=F$, \ie initial and final states do not change;
  \item the (partial) transition function $\tau^a$ is defined as follows. Let
    $w=\tau(q,\event)$ with $\event=(\actions,\delta)$. We distinguish the case where $\delta=1$ or $\delta>1$.
    \begin{enumerate}
      \item if $\delta=1$, let $\actions_C\subseteq\actions$ and
      $\actions_E\subseteq\actions$ be the set of actions in $\actions$ playable
      by \charlie and by \eve, respectively. Then:
      \begin{enumerate}
        \item $\tau(q,\play(\actions_C^e))=q_{1,\actions_C^e}$, where 
          $\actions_C^e$ is the set of \emph{ending} actions in $\actions_C$;
        \item $\tau(q_{1,\actions_C^e},\play(\actions_E^e))=q_{1,\actions_C^e\cup\actions_E^e}$, where 
          $\actions_E^e$ is the set of \emph{ending} actions in $\actions_E$;
        \item $\tau(q_{1,\actions_C^e\cup\actions_E^e},\play(\actions_C^s))=q_{1,\actions_C^e\cup\actions_E^e\cup\actions_C^s}$, where 
          $\actions_C^s$ is the set of \emph{starting} actions in $\actions_C$;
        \item $\tau(q_{1,\actions_C^e\cup\actions_E^e\cup\actions_C^s},\play(\actions_E^s))=w$, where 
          $\actions_E^s$ is the set of \emph{starting} actions in $\actions_E$;
      \end{enumerate}
      where the mentioned states are added to $Q^a$ as needed.
      \item if $\delta>1$, let $\actions_C\subseteq\actions$ and
      $\actions_E\subseteq\actions$ be the set of actions in $\actions$ playable
      by \charlie and by \eve, respectively. Note that by construction,
      $\actions_C$ only contains \emph{starting} actions. Then:
      \begin{enumerate}
        \item $\tau(q,\wait(\delta_C))=q_{\delta_C}$ for all 
          $\delta\le\delta_C\le d$;
        \item $\tau(q_{\delta_C},\play(\delta, \actions_E^e))=q_{\delta,\actions_E^e}$
          where $\actions_E^e$ is the set of \emph{ending} actions in
          $\actions_E$;
        \item $\tau(q_{\delta,\actions_E^e},\play(\actions_C))=q_{\delta,\actions_E^e\cup\actions_C}$;
        \item $\tau(q_{\delta,\actions_E^e\cup\actions_C},\play(\actions_E^s))=w$ where
          $\actions_E^s$ is the set of \emph{starting} actions in $\actions_E$;
      \end{enumerate}
      where the mentioned states are added to $Q^a$ as needed.
    \end{enumerate}
    All the transitions not explicitly defined above are undefined.
\end{enumerate}

A graphical example of the above construction can be seen in
\cref{fig:constructions}, on the right. Note that the structure of the original
$A_G$ automaton is preserved by $A^a_G$. In particular, one can see that for
each $q\in Q$ and event $\event=(A,\delta)$, any sequence of moves whose outcome
would append $\event$ to the partial plan (see \cref{def:games:round-outcome})
reach from $q$ the same state $w$ in $A^a_G$ that is reached in $A_G$ by reading
$\event$. Hence, one can consider $A^a_G$ to also being able to \emph{read}
event sequences, even though its alphabet is different. We denote as $[\evseq]$
the state $q\in Q^a$ reached by reading $\evseq$ in $A^a_G$.

Moreover, note that, with a minimal abuse of notation, any play $\bar\rho$ for
the game $G$ can be seen as a word readable by the automaton $A_G^a$. Hence, we
can prove the following.
\begin{theorem}
  \label{thm:arena-soundness}
  If $G$ is a timeline-based game, for any play $\bar\rho$ for $G$, $\bar\rho$
  is successful if and only if it is accepted by $A_G^a$.
\end{theorem}

\subsection{Computing the Winning Strategy}

Once built the arena, we can focus on computing the winning region $W_C$
for \charlie, that is, the set of states of the arena from which \charlie
can force the play to reach a final state of $A_G^a$, no matter of the
strategy of \eve. These games are called \emph{reachability
games}~\cite{Thomas2008}.  If the winning region $W_C$ is not empty,
a winning strategy of \charlie can be simply derived from $W_C$.  As
a consequence of \cref{thm:soundness-completeness,thm:arena-soundness}, the
computed winning strategy $\sigma_C$ for $A_G^a$ respects
\cref{def:games:winning-strategy}.

As stated in \cite[Theorem 4.1]{Thomas2008}, rechability games are
determined, and the winning region $W_C$ along with the corresponding
positional winning strategy $s$ are computable. Let $A_G^a=(Q^a,\Sigma^a,
q_0^a, F^a, \tau^a)$ be the automaton built from $G$ as described in the
previous section. Note that, by construction, in any state $q\in Q^a$ only
one of the players has available moves. Let $Q^a_C\subseteq Q^a$ be the set
of states \emph{belonging} to \charlie, \ie states from which \charlie can
move, and let $Q^a_E=Q^a\setminus Q^a_C$. Moreover, let $E=\set{ (q, q')
\in Q^a\times Q^a \mid \exists \event \mathrel{.} \tau^a(q, \event) = q'}$,
\ie the set of all the edges of $A_G^a$. 

Now, for each $i\ge0$, we can compute the $i$-th attractor of $F^a$, written $Attr^i_C(F^a)$, that is, the set of states from which \charlie can win in at most $i$ steps. $Attr^i_C(F^a)$ is defined as follows:
\begin{align*} Attr^0_C (F^a) &= F^a \\
Attr^{i+1}_C (F^a) &= Attr^i_C (F^a) \\
&\cup \set{ q^a \in Q^a_C \, | \, \exists r \big((q^a, r) \in E \land r \in Attr^i_C (F^a)\big) } \\
&\cup \set{ q^a \in Q^a_E \, | \, \forall r \big((q^a, r) \in E \implies r \in Attr^i_C (F^a) \big) }
\end{align*}
As remarked in \cite{Thomas2008}, the sequence $Attr^0_C (F^a) \subseteq
Attr^1_C (F^a) \subseteq Attr^2_C (F^a) \subseteq \ldots$ becomes
stationary for some index $k \leq \lvert Q^a \rvert$. Thus, we define
$Attr_C (F^a) = \bigcup^{\lvert Q^a \rvert}_{i=0} Attr^i_C(F^a)$. 
In order to prove that $W_C = Attr_C(F^a)$, it suffices to use the proof of
\cite[Theorem 4.1]{Thomas2008} for showing that $Attr_C(F^a) \subseteq W_C$
and $W_C \subseteq Attr_C(F^a)$.

To compute a winning strategy for \charlie in the case that $q_0^a\in W_C$,
it is sufficient to define $s (q) = \mu$ for any $\mu$ such that
$\tau^a(q,\mu)=q'$ with $q,q'\in W_C$ (which is guaranteed to exist by
construction of the attractor). Then, the strategy $\sigma_C$ for \charlie in $G$ (see \cref{def:games:winning-strategy}) is defined as $\sigma_C(\evseq)
= s([\evseq])$. 

\begin{theorem}
    \label{thm:winning-region-soundness-completeness}
  Given $A_G^a=(Q^a,\Sigma^a, q_0^a, F^a, \tau^a)$, $q_0^a\in W_C$ if and only
  if \charlie has a winning strategy $\sigma_C$ for $G$.
\end{theorem}
\begin{proof}
  We first prove \emph{soundness}, that is, $q_0^a\in W_C$ implies that
  \charlie has a winning strategy $\sigma_C$ for $G$. If $q_0^a\in W_C$,
  then it means that there exists a positional winning strategy $s$ for
  \charlie for the reachability game over the arena $A_G^a$. By
  \cref{thm:arena-soundness} and by the definition of reachability game, we
  know that each play generated by $s$ corresponds to a successful play for
  game $G$. 
  Let $\sigma_C(\evseq)=s([\evseq])$ be the winning strategy for \charlie
  in game $G$ as defined above. By construction of $\sigma_C$ and by
  \cref{def:games:winning-strategy}, this means that $\sigma_C$ is
  a winning strategy of \charlie for $G$.

  To prove \emph{completeness} (\ie if \charlie has a winning strategy
  $\sigma_C$ for $G$ then $q_0^a\in W_C$), we proceed as follows. From
  \cref{def:games:winning-strategy} we know that a winning strategy
  $\sigma_C$ for \charlie is a strategy such that for every admissible
  strategy $\sigma_E$ for \eve, there exists $n \ge 0$ such that the play
  $\rounds_n(\strategy_C,\strategy_E)$ is successful.  From
  \cref{thm:arena-soundness}, we know that
  $\rounds_n(\strategy_C,\strategy_E)$ is accepted by $A^a_G$.  Therefore,
  $\rounds_n(\strategy_C,\strategy_E)$ reaches a state in the set
  $F^a$ starting from $q^a_0$. By definition of reachability game, this
  means that $q^a_0 \in W_C$.
\end{proof}


\section{Conclusions}
\label{sec:conclusions}

In this paper, we completed the picture about timeline-based games by providing
an effective procedure for controller synthesis, whereas before only a proof of
the complexity of the strategy existence problem was known. Previous approaches
either provided a deterministic concurrent game structure which was however not
built effectively, or an effectively built automata which was, however,
nondeterministic and thus unsuitable for use as a game arena without a costly
determinization. Our approach surpasses the limits of both previous ones by
providing a deterministic construction, of optimal asymptotic size, suitable to
be used as a game arena. Then, we solve the reachability game on the arena with
standard methods to effectively compute the winning strategy for the game, if it
exists.

This work paves the way to interesting future developments. On the one hand, the
effective procedure shown here can be finally implemented, bringing
timeline-based games from theory to practice. On the other hand, developing an
effective system based on such games requires to answer many interesting
questions, from which concrete modeling language to adopt, to which algorithmic
improvements are needed to make the approach feasible. For example, it can be
foreseen that, to solve the fixpoint computation that leads to the strategy with
reasonable performance, the application of \emph{symbolic techniques} would be
needed.

\section*{Acknowledgements}
Nicola Gigante and Luca Geatti acknowledge the support of the Free University of
Bozen-Bolzano, Faculty of Computer Science, by means of the projects TOTA
(\emph{Temporal Ontologies and Tableaux Algorithms}) and STAGE (\emph{Synthesis
of Timeline-based Planning Games}).

\bibliographystyle{eptcs}
\bibliography{biblio}

\end{document}